\documentclass[letterpaper]{article}
\usepackage{uai2019}
\usepackage[margin=1in]{geometry}

\usepackage[authoryear]{natbib}

\usepackage{algorithm,algorithmic}
\usepackage[utf8]{inputenc}
\usepackage[T1]{fontenc}
\usepackage{amsmath,amssymb,amsthm}
\usepackage{mathtools}
\usepackage{xcolor}
\usepackage{verbatim}
\usepackage{natbib}
\usepackage{bbm}
\usepackage{times}
\usepackage{dsfont}
\allowdisplaybreaks

\newtheorem{theorem}{Theorem}
\newtheorem{corollary}[theorem]{Corollary}

\newtheorem{lemma}[theorem]{Lemma}

\newcommand{\steps}{T}
\newcommand{\rh}{\hat{r}}
\newcommand{\ph}{\hat{p}}

\newcommand{\rup}{\bar{r}}
\newcommand{\defined}{\coloneqq}

\newcommand{\rt}{\tilde{r}}
\newcommand{\pt}{\tilde{p}}
\newcommand{\Vt}{\tilde{V}}
\newcommand{\pit}{\tilde{\pi}}

\newcommand{\rhot}{\tilde{\rho}}

\newcommand{\MDP}{M}

\newcommand{\bP}{\mathbb{P}}
\newcommand{\bE}{\mathbb{E}}
\newcommand{\set}[1]{\mathcal{#1}}
\newcommand{\sS}{\set{S}}
\newcommand{\sA}{\set{A}}
\newcommand{\hrs}[3]{\ensuremath{\hat r_{#1}\left ( {#2}, {#3} \right ) }}
\newcommand{\hpx}[4]{\ensuremath{\hat p_{#1}\left ( {#2} | {#3}, {#4} \right ) }}

\newcommand{\bigO}{\mathcal{O}}
\newcommand{\ind}{\mathds{1}}

\title{Variational Regret Bounds for Reinforcement Learning}

\author{ 
Ronald Ortner  \\
rortner@unileoben.ac.at \\
\And
Pratik Gajane \\
pratik.gajane@unileoben.ac.at\\
Lehrstuhl für Informationstechnologie, Montanuniversität Leoben, Austria
\And
Peter Auer  \\
auer@unileoben.ac.at  
}
\begin{document}

\maketitle

\begin{abstract}
We consider undiscounted reinforcement learning in Markov decision processes (MDPs) where \textit{both} the reward functions and the state-transition probabilities may vary (gradually or abruptly) over time.  For this problem setting, we propose an algorithm and provide performance guarantees for the regret evaluated against the optimal non-stationary policy. The upper bound on the regret is given in terms of the total \textit{variation} in the MDP. This is the first variational regret bound for the general reinforcement learning setting. 
\end{abstract}

\section{INTRODUCTION} 
A Markov decision process (MDP) is a discrete-time state-transition system in which the transition dynamics follow the Markov property (\citet{Puterman1994}, \citet{Bertsekas:1996:NP:560669}). MDPs are a standard model to express uncertainty in reinforcement learning problems. In the classical MDP model, the transition dynamics and the reward functions are time-invariant.  
However, such fixed  transition dynamics and reward functions are insufficient to model real world problems in which parameters of the world change over time. To deal with such problems, we consider a setting in which both the transition dynamics and the reward functions may vary over time. These changes can be either abrupt or gradual. 
As a motivation, consider the problem of deciding which ads to place on a webpage. The instantaneous reward is the payoff when viewers are redirected to an advertiser, and the state captures the details of the current ad. With a heterogeneous group of viewers, an invariant state-transition function cannot accurately capture the transition dynamics. The instantaneous reward, dependent on external factors, is also better represented by changing reward functions.
For additional motivation and further applications, see \citep{Yu2009a,Yu2009b,NIPS2013_4975}.

\subsection{MAIN CONTRIBUTION}
For reinforcement learning in MDPs with changes in reward functions and transition probabilities, we provide an algorithm, Variation-aware UCRL, a variant of UCRL with restarts \citep{jaksch}, which restarts according to a schedule dependent on the \textit{variation} in the MDP (defined in Section \ref{sec:Setting} below). 
For reinforcement learning in an MDP with $S$~states, $A$~actions, diameter~$D$, and changes with a variation of $V$ we derive for our algorithm a high-probability upper bound on the cumulative regret after $T$ steps of 
$\tilde{\bigO}( V^{1/3} T^{2/3} D S \sqrt{A})$. This bound is optimal with respect to time~$T$ and variation~$V$ and improves the known regret bound of $\tilde{\bigO}(L^{1/3} T^{2/3}D S \sqrt{A})$ for UCRL with restarts
in MDPs with $L$ abrupt changes, when using a restart schedule dependent on $L$ \citep{jaksch}. In case when reward functions and transition probabilities change gradually, the latter bound becomes trivial when $L$ is of order $T^{1/3}$, while our bound is still sublinear as long as the variation is sufficiently small. To the best of our knowledge, our bounds are the first variational bounds for the general reinforcement learning setting. So far, variational regret bounds have been derived only for  simpler bandit settings \citep{NIPS2014_5378,luo19}.

\subsection{RELATED WORK}
\citet{Nilim:2005:RCM:1246500.1246504} consider MDPs with arbitrarily changing state-transition probabilities but fixed reward functions where it is assumed that the uncertainty in the transition probabilities is state-wise independent.
They provide a robust dynamic programming algorithm and prove that it is optimal with respect to the worst-case performance in terms of the expected total cost.

\citet{NIPS2004_2730} and \citet{Dick2014} consider the problem of MDPs with fixed state-transition probabilities and changing reward functions and measure the performance of the learner against the best stationary policy in hindsight.
\citet{NIPS2004_2730} assume that the learner has complete knowledge of all the previous reward functions (i.e., also for states not visited) and provide regret bounds which depend on the mixing time.
\citet{Dick2014} model learning in MDPs as an online linear optimization problem and propose solutions based on variants of mirror-descent.

\citet{Yu2009a} and \citet{Yu2009b} consider arbitrary changes in the reward functions and arbitrary, but bounded, changes in the state-transition probabilities.
They also give regret bounds that scale with the proportion of changes in the state-transition kernel and which in the worst case grow linearly with time.

\citet{NIPS2013_4975} consider MDP problems with (oblivious) adversarial changes in state-transition probabilities and  reward functions and provide an algorithm which guarantees $O(\sqrt{T})$ regret with respect to a comparison set of stationary (expert) policies.

\iffalse
\subsection{OUTLINE}
The rest of the article is structured as follows. In Section~\ref{sec:Setting}, we describe the problem at hand formally. In Section \ref{sec:Algo}, we present our algorithm. The main result is given in Section~\ref{sec:Results}. In Section~\ref{sec:Proofs}, we provide proofs for the two main theorems as well as other preliminary results we require. In Sections \ref{sec:Discussion} and \ref{sec:conc}, we discuss extensions and possible future directions.  
\fi

\section{SETTING}
\label{sec:Setting}
We start with collecting some basic facts about Markov decision processes (MDPs).
In a (time-homogeneous) MDP $M=(\sS,\sA,\rup,p,s_1)$ with a set $\sS$ of $S$ states, a set $\sA$ of $A$ actions the learner starts in some initial state~$s_1$. At each time step $t=1,2,\ldots$ she chooses an action~$a_t=a$ in the current state~$s_t=s$, receives a random reward~$r_t$ with mean $\rup(s,a)$ and observes a transition to the next state~$s_{t+1}=s'$ according to transition probabilities $p(s'|s,a)$. Note that in a time-homogeneous MDP mean rewards and transition probabilities only depend on the current state and the chosen action.

An MDP is called \textit{communicating}, if for any two states~$s$,~$s'$, when starting in $s$ it is possible to reach $s'$ with positive probability choosing appropriate actions. In communicating MDPs we define the \textit{diameter} to be the minimal expected time it takes to get from any state to any other state in the MDP, cf.~\citep{jaksch}.

For acting in an MDP one usually considers stationary policies $\pi:\sS\to\sA$ that fix for each state~$s$ the action~$\pi(s)$ to choose. The average reward $\rho(\MDP, \pi)$ of a stationary policy $\pi$ is the limit of the expected average accumulated reward when following $\pi$, i.e.,
$$
\rho(\MDP, \pi) \defined \lim_{T \rightarrow \infty} \frac{1}{T} \, \bE \left[ \sum_{t=1}^{T} r_t \right].
$$
The optimal average reward $\rho^*(\MDP)=\max_{\pi} \rho(\MDP, \pi)$ in communicating MDPs is independent of the initial state~$s_1$ and cannot be increased when using nonstationary policies \citep{Puterman1994}.

In the problem setting we consider the underlying MDP is not time-homogeneous. Rather the mean rewards and transition probabilities depend on the current step $t$. Accordingly, we write them as $\rup_t(s,a)$ and $p_t(s'|s,a)$, respectively, and denote the (time-homogeneous) MDP at step $t$ by $\MDP_t=(\sS,\sA,\rup_t,p_t,s_1)$. We assume that all MDPs $M_t$ are communicating with diameter $D_t$ and denote by $D$ a common upper bound on all $D_t$. 
%Additionally, we assume that at any step $t$, any two states can be connected in expectedly $\leq D$ steps in the underlying non-time-homogeneous MDP as well. (It is not difficult to give examples that show that this need not be the case even if each MDP $M_t$ at step $t$ has diameter $\leq D$.)

Obviously, in a nonstationary MDP the optimal policy in general will not be stationary anymore. We are interested in online regret bounds after any $T$ steps taken by the learner. Accordingly, we consider the optimal expected $T$-step reward $v^*_T(s_1)$ that can be achieved by any (time dependent) policy when starting in $s_1$, and define the regret after $T$ steps as 
\[
   R_T := v^*_T(s_1) - \sum_{t=1}^T  r_t .
\]
If there are no changes, this basically corresponds to the standard notion of regret as used e.g.\ by \cite{jaksch} (apart from an additive constant of order $D$, cf.\ footnote~1 ibid.). 
In the following, we assume that the random rewards $r_t$ are always bounded in $[0,1]$.

\subsection{DEFINITION OF VARIATION}
We consider individual terms for the \textit{variation} in mean rewards and transition probabilities, that is,
\begin{eqnarray*}
     V^r_T &:=& \sum_{t=1}^{T-1} \max_{s,a} \big| \rup_{t+1}(s,a) - \rup_{t}(s,a) \big|, \mbox{ and } \\
     V^p_T &:=& \sum_{t=1}^{T-1} \max_{s,a} \big\| p_{t+1}(\cdot|s,a) - p_{t}(\cdot|s,a) \big\|_1. 
\end{eqnarray*}
These ``local'' variation measures can also be used to bound a more ``global'' notion of variation in average reward defined as
\[
    V_T := \sum_{t=1}^{T-1} \big| \rho^*(M_{t+1}) - \rho^*(M_{t}) \big|. 
\]

\begin{theorem}
\label{thm:PertBound}
 $V_T \leq V^r_T + D V^p_T$.
\end{theorem}
The proof of Theorem~\ref{thm:PertBound} is given in Section \ref{sec:pert} below. 
As an example of \cite{Ortner2014} shows, the bound of Theorem~\ref{thm:PertBound} is best possible.

While $V_T$ is a more straightforward adaptation of the notion of variation of \cite{NIPS2014_5378} from the bandit to the MDP setting, in the latter it seems more natural to work with the local variation measures for rewards and transition probabilities, as the learner does not have direct access to the average rewards of policies.

%%%%
\section{ALGORITHM}
\label{sec:Algo}
%%%%
For reinforcement learning in the changing MDP setting, we propose Variation-aware UCRL (shown as Algorithm~\ref{alg}), which is based on the UCRL algorithm of \cite{jaksch}.

\begin{algorithm}%[!ht]
\caption{Variation-aware UCRL}\label{alg}
\begin{algorithmic}[1]
\STATE \textbf{Input:} States $\sS$, actions ~$\sA$, confidence parameter~$\delta$, variation parameters $\Vt^r$, $\Vt^p$ for rewards and transition probabilities.
\STATE \textbf{Initialization:} Set current time step $t:=1$.
\FOR {episode $k=1, \dots$} 
\STATE Set episode start $t_k:=t$.
\STATE Let $v_k(s,a)$ denote the state-action counts for visits in current episode $k$, and 
$N_k(s,a)$ be the counts for visits before episode $k$.
\STATE\label{algo:setP}\label{algo:compEst}
               For $s,s'\in\sS$ and $a\in\sA$, compute estimates 
               \begin{align*}
                 \hrs{k}{s}{a} &:= \frac{\sum_{\tau} r_{\tau} \cdot \ind_{s_\tau=s, a_\tau=a}}{\max(1,N_k(s,a))},
                    \\
                 \hpx{k}{s'}{s}{a} &:= \frac{\#\big\{
                \tau: s_\tau = s, a_\tau = a, s_{\tau+1} = s^\prime
                         \big\}}{\max(1,N_k(s,a))}.
               \end{align*}
\textbf{Compute policy $\pit_k:$}
\STATE \label{alg:pl} Let $\mathcal{M}_k$ be the set of plausible MDPs $\tilde{M}$ with rewards $\tilde{r}(s,a)$ and transition probabilities $\tilde{p}(\cdot|s,a)$ satisfying 
%\begin{align}
%\hspace{-0.8cm}|\tilde{r}(s,a) - \hat{r}_k(s,a)| &\leq \Vt^r + \sqrt{\tfrac{8\log{(8SAt_k^3/\delta)}}%{\max{(1, N_k(s,a))}}}, \label{eq:ConfReward} \\
%\hspace{-0.8cm}  \big\| \tilde{p}(\cdot|s,a) - \hat{p}_k(\cdot|s,a) \big\|_1 &\leq \Vt^p + \sqrt{\tfrac{8 S \log{(8SAt_k^3/\delta)}}{\max{(1, N_k(s,a))}}}  \label{eq:ConfProb}.
%\end{align}
\begin{align}
&\hspace{-0.8cm}|\tilde{r}(s,a) - \hat{r}_k(s,a)| \nonumber\\&\leq \Vt^r + \sqrt{\tfrac{8\log{(8SAt_k^3/\delta)}}{\max{(1, N_k(s,a))}}}, \label{eq:ConfReward} \\
&\hspace{-0.8cm}  \big\| \tilde{p}(\cdot|s,a) - \hat{p}_k(\cdot|s,a) \big\|_1 \nonumber\\ &\leq \Vt^p + \sqrt{\tfrac{8 S \log{(8SAt_k^3/\delta)}}{\max{(1, N_k(s,a))}}}  \label{eq:ConfProb}.
\end{align}
\STATE \label{alg:policy} Use extended value iteration (see Section 3.1.2 of \cite{jaksch}) to find a policy $\pit_k$ and an optimistic MDP $\tilde{M}_k \in \mathcal{M}_k$ such that 
$$
\rhot_k \defined \rho(\tilde{M}_k,\pit_k) = \max_{M' \in \mathcal{M}_k} \rho^*(M'). %- \frac{1}{\sqrt{t_k}}
$$ \\ 
\textbf{Execute policy $\pit_k$:}
\STATE \label{alg:vole} \textbf{while} $v_k(s_t,\pit_k(s_t)) < \max(1,N_k(s_t, \pit_k(s_t)))$ \textbf{do} %\smallskip
\STATE  \label{alg:le} Choose action $a_t = \tilde{\pi}_k(s_t)$, obtain reward $r_t$, and observe $s_{t+1}$. Set $t = t + 1$.
\ENDFOR
\end{algorithmic}
\end{algorithm}

UCRL is an algorithm that is based on the idea of being optimistic in the face of uncertainty. 
It maintains estimates of rewards and transition probabilities (line~\ref{algo:compEst}) and employs confidence intervals to define a set~$\mathcal M$ of MDPs that are plausible with respect to the observations so far (line~\ref{alg:pl}). When computing a new policy the algorithm chooses the policy $\pit$ and the MDP $\tilde{M}$ in $\mathcal M$ that give the highest average reward (line~\ref{alg:policy}). UCRL employs this optimistic policy $\pit$ until the state-action visits in some state have doubled (lines \ref{alg:vole}--\ref{alg:le}), when a new policy is computed. The time intervals in which the policy is fixed are called \textit{episodes}.

For the changing MDP setting, we use adapted confidence intervals \eqref{eq:ConfReward} and \eqref{eq:ConfProb} to account for the variation in rewards and transition probabilities. The arising algorithm basically corresponds to the colored UCRL2 algorithm suggested by \cite{restless} for reinforcement learning in MDPs with given similarities.

While the regret of Variation-aware UCRL contains a term that is linear in the number of steps (cf.\ Theorem~\ref{thm:wr} below), 
we can obtain sublinear regret bounds by restarting the algorithm according to a suitable scheme shown as Algorithm~\ref{alg2}.
Our restart schedule is optimized with respect to the variation, as the regret bounds presented in the next section will show. Note that the algorithm needs (upper bounds on) the local variations $V_T^r$ and $V_T^p$ as an input. Alternatively, an upper bound on the global variation $V_T$ (replacing the term $V_T^r+V_T^p$ in the algorithm) could be used as well.
As the bound of Theorem~\ref{thm:PertBound} is best possible, this gives regret bounds that are worse by a factor up to $D^{1/3}$ however.
\begin{algorithm}%[!ht]
\caption{Variation-aware UCRL with restarts}\label{alg2}
\begin{algorithmic}[1]
\STATE \textbf{Input:}  States $\sS$, actions ~$\sA$, confidence parameter $\delta$, variation terms $V_T^r$ and $V_T^p$.
\STATE \textbf{Initialization:} Set current time step $\tau:=1$.
\FOR {$\mbox{phase }i=1, \dots$} 
\STATE Perform UCRL with confidence parameter $\delta/2\tau^2$ \\ \quad for $\theta_i:=\Big\lceil\frac{i^2}{(V_T^r+V_T^p)^2}\Big\rceil$ steps.
\STATE Set $\tau = \tau + \theta_i$.
\ENDFOR
\end{algorithmic}
\end{algorithm}

The idea of restarting UCRL in the changing MDP setting has already been considered by \cite{jaksch}. When a bound $L$ on the total number of changes is known, then using a restart schedule adapted to $L$ gave the following regret bound.\footnote{\cite{jaksch} consider a slightly different notion of regret 
defined as $\sum_t (\rho_t^*-r_t)$, where $\rho_t^*:=\rho^*(M_t)$ is the optimal average reward 
at step $t$. However, when there are at most $L$ changes, the difference to our notion of regret is only of order $LD$.}
%\pagebreak

\begin{theorem}[\cite{jaksch}]\label{thm:regret-ucrl-r}
In an MDP with at most $L$ changes, after $T$ steps the regret of UCRL restarted with confidence parameter $\frac{\delta}{L^2}$ at steps $\big\lceil \tfrac{i^3}{(L+1)^2}\big\rceil$ for $i=1,2,3,\dots$ is upper bounded as 
\begin{eqnarray*}
  R_T \,\leq\,65 \cdot (L+1)^{1/3}\, T^{2/3}  D S \sqrt{ A\log \left(\tfrac{T}{\delta}\right)}
\end{eqnarray*}
with probability of at least $1-\delta$.
\end{theorem}

Note that the restart schedule of Theorem \ref{thm:regret-ucrl-r} basically corresponds to performing UCRL for $\sim \frac{i^2}{L^2}$ steps for $i=1,2,\ldots$, which is similar to our algorithm replacing $L$ by the variation term $V_T^r + V_T^p$.

%%%%
\section{MAIN RESULT}
\label{sec:Results}
%%%%

The following regret bound for Variation-aware UCRL with restarts is our main result.
The proof is given in the next section.

\begin{theorem}
\label{thm:vregret-ucrl-r}
After any $T$ steps, the regret of the restart scheme for Variation-aware UCRL of Algorithm~\ref{alg2} is bounded as
\begin{align*}
R_T &\leq 74 \cdot  (V_T^r + V_T^p)^{1/3}\, T^{2/3} DS \sqrt{A\log{\big(\tfrac{16S^2AT^5}{\delta})}} 
\end{align*}
with probability $1-\delta$, provided that in each phase $i$ the variation parameters $\Vt^r_i$,  $\Vt^p_i$ are set to the respective true variation values for phase $i$. 
\end{theorem}

If there are $L$ changes, the bounds of Theorems~\ref{thm:regret-ucrl-r} and~\ref{thm:vregret-ucrl-r} are of the same order.  On the other hand, the bound of Theorem~\ref{thm:vregret-ucrl-r} is better than the bound of Theorem~\ref{thm:regret-ucrl-r}, when $L$ is large but the variation small as e.g.\ when having small gradual changes at any time step. Thus, Theorem~\ref{thm:vregret-ucrl-r} can be considered as an improvement over Theorem~\ref{thm:regret-ucrl-r}.

%is worse by a factor of $D^{1/3}$. 
%comparable to that of Theorem~\ref{thm:regret-ucrl-r} with $L$ replaced by the diameter.

With respect to the variation and the horizon, our bound is optimal, as bounds of $\tilde{\bigO}(V^{1/3}T^{2/3})$ are already best possible in the bandit setting \citep{NIPS2014_5378}.

%%%%
\section{ANALYSIS}
\label{sec:Proofs}
%%%%
%

%\subsection{PRELIMINARIES}\label{sec:prel}

We start with some preliminaries. First, we introduce the Poisson equation for the optimal policy in a communicating MDP. That is, the mean rewards $\rup(s,\pi^*(s))$ under an optimal policy $\pi^*$ and the respective optimal average reward $\rho^*$ are related via the Poisson equation $\rho^*(\MDP') - \rup(s,\pi^*(s)) = \sum_{s'} p'(s'|s,\pi^*(s))\cdot \lambda(s') - \lambda(s)$, where $\lambda$ is the so-called bias function for $\pi^*$, cf.\ \citep{Puterman1994}. It holds that each $\lambda(s)$ as well as the span of the bias function $\Lambda:=\max_s \lambda(s) - \min_{s'} \lambda(s)$ is upper bounded by the diameter, cf.\ \citep{jaksch,regal}.

We will frequently make use of Azuma-Hoeffding inequality, which we state here for convenience.
\begin{lemma}[\textbf{Azuma-Hoeffding inequality \citep{Hoeffding:1963}}]
\label{azuma}
Let $X_1, X_2, \dots$ be a martingale difference sequence with $|X_i| \leq c$ for all $i$. Then for all $\epsilon > 0$ and $n \in \mathbb{N}$,\vspace{-3mm}
$$
\bP \Big\{ \sum_{i=1}^n X_i \geq \epsilon \Big\} \leq \exp{\Big(-\frac{\epsilon^2}{2nc^2}\Big)} .
$$
\end{lemma}

\subsection{A PERTURBATION BOUND}\label{sec:pert}

We continue with establishing a perturbation bound on the optimal average reward, which is a generalization of Lemma~8 of \citet{Ortner2014}. %For the sake of completeness, we give a proof in 
\begin{lemma}
\label{lem:ProofPertBound}
Assume we have two MDPs $\MDP=(\sS,\sA,\rup_t,p_t,s_1), \MDP'=(\sS,\sA,\rup',p',s_1)$ on the same state and action space. The MDP $M$ may be non-time-homogeneous so that its mean rewards $\rup_t$ and transition probabilities $p_t$ are allowed to depend on time $t$. We assume that $M'$ is time-homogeneous and communicating with optimal policy $\pi'^*$, such that for all steps~$t=1,\ldots,T$,
\begin{align*}
\max_{s}  \big|\rup_t(s,\pi'^*(s)) - \rup'(s,\pi'^*(s))  \big| &\leq \Delta^r_t(s) ,   \\
\max_{s}  \big\|p_t(\cdot|s,\pi'^*(s)) - p'(\cdot|s,\pi'^*(s)) \big\|_1 &\leq \Delta^p_t(s).
\end{align*}
If $\pi'^*$ is performed on $\MDP$ for $\steps$ steps, then denoting by~$s_t$ the state visited at step $t$ it holds that
\begin{align*}
&  \steps \rho^*(\MDP') - \sum_{t=1}^T \rup_t(s_t, \pi'^*(s_t))  \\
& \quad\leq   \sum_{t=1}^T \big( \Lambda' \Delta^p_t(s_t) + \Delta^r_t(s_t) \big) \\
& \qquad +  \sum_{t=1}^T \Big(\sum_{s'} p_t(s'|s_t)\cdot \lambda'(s') - \lambda'(s_t) \Big) ,   
\end{align*}
where $\lambda'$ is the bias function and $\Lambda'$ the respective bias span of $\pi'^*$ on $M'$.
\end{lemma}

\begin{proof}
The proof is a modification of the proof of Lemma 8 in Appendix A of \citep{Ortner2014}. 
Abbreviating $\rup_t(s):=\rup_t(s,\pi'^*(s))$, $\rup'(s):=\rup'(s,\pi'^*(s))$ and $p_t(s'|s):=p_t(s'|s,\pi'^*(s))$, $p'(s'|s):=p'(s'|s,\pi'^*(s))$ in the following, we can write
\begin{align}
&  \steps \rho^*(\MDP') - \sum_{t=1}^T \rup_t(s_t) 
 =  \sum_{t=1}^T \big( \rho^*(\MDP') - \rup_t(s_t)\big)  \nonumber \\
& \quad  \leq \,  \sum_{t=1}^T \big( \rho^*(\MDP') - \rup'(s_t)\big)  
			+  \sum_{t=1}^T \big(\rup'(s_t) - \rup_t(s_t) \big) \nonumber  \\
& \quad  \leq \,  \sum_{t=1}^T \big( \rho^*(\MDP') - \rup'(s_t)\big)  
			+  \sum_{t=1}^T \Delta^r_t(s_t). \label{eq:m2}
\end{align}
For bounding the first term in \eqref{eq:m2} we use that by the Poisson equation for policy $\pi'^*$ on $\MDP'$ we have that
$\rho^*(\MDP') - \rup'(s) = \sum_{s'} p'(s'|s)\cdot \lambda'(s') - \lambda'(s)$. Accordingly, it holds that
\begin{align*}
 &  \sum_{t=1}^T \big( \rho^*(\MDP') - \rup'(s_t)\big)   \nonumber \\
 &  =  \sum_{t=1}^T \Big(\sum_{s'} p'(s'|s_t)\cdot \lambda'(s') - \lambda'(s_t) \Big)  \nonumber\\
 &  \leq  \sum_{t=1}^T \Big(\sum_{s'} p_t(s'|s_t)\cdot \lambda'(s') - \lambda'(s') \Big)  \nonumber\\
 & \quad +  \sum_{t=1}^T \sum_{s'} \big( p'(s'|s_t) - p_t(s'|s_t) \big)\lambda'(s_t)   \\
 & \leq  \sum_{t=1}^T \!\Big(\sum_{s'} p_t(s'|s_t)\!\cdot\! \lambda'(s') - \lambda'(s_t) \Big)  
\! + \!\sum_{t=1}^T \Lambda'\Delta^p_t(s_t), 
\end{align*}
whence the lemma follows together with \eqref{eq:m2}.
\end{proof}

For the analysis of the last term in the bound of Lemma~\ref{lem:ProofPertBound} we can use the following result, which is a simple generalization of a technique used by \cite{jaksch}.
\begin{lemma}\label{lem:azuma-applied}
Consider some MDP $\MDP=(\sS,\sA,\rup,p,s_1)$ and let $f:\sS\to\mathbb{R}$ be some function on the state space of~$M$.
Then for any (possibly nonstationary) policy choosing at each step $\tau$ action $a_\tau$ in the current state~$s_\tau$,
it holds with probability at least $1 - \delta$,
\begin{align*}
 &  \sum_{\tau=1}^T \Big(\sum_{s'} p(s'|s_\tau,a_\tau)\cdot f(s') - f(s_\tau) \Big)   \\
 &  \quad \leq F \sqrt{2\steps\log{\big(\tfrac{1}{\delta}\big)}} + F, 
\end{align*}
where $F:=\max_s f(s) - \min_s f(s)$ is the span of $f$.
\end{lemma}
\begin{proof}
Following an argument due to \cite{jaksch} we write
\begin{align}
 &  \sum_{\tau=1}^T \Big(\sum_{s'} p(s'|s_\tau)\cdot f(s') - f(s_\tau) \Big) \nonumber \\
 & =  \sum_{\tau=1}^T \Big(\sum_{s'} p(s'|s_\tau)\cdot f(s') - f(s_{\tau+1}) \Big) \nonumber \\
 & \qquad       +   f(s_{T+1}) - f(s_{1}) .
\end{align}
Now $f(s_{T+1}) - f(s_{1}) \leq F$, while the sum is a martingale difference sequence $\sum_\tau X_{\tau}$ with $|X_\tau|\leq F$. The lemma follows by an application of Azuma-Hoeffding (Lemma~\ref{azuma}).
%we obtain that with probability $1-\delta$,
%\begin{align*}
% &  \sum_{\tau=1}^T \Big(\sum_{s'} p(s'|s_\tau)\cdot f(s') - f(s_\tau) \Big)  \leq F\sqrt{2 \steps \log{\big(\tfrac{1}{\delta}\big)}} + F. \qedhere %\label{eq:m6}
%\end{align*}
\end{proof}

The following corollary is a variant of Lemma 9 contained in the (unpublished) appendix of \citep{Ortner2014}.\footnote{We note that a bound on the absolute value of the difference in average reward as stated in \citep{Ortner2014} will in general depend on the bias spans resp.\ the diameters of \textit{both} MDPs, that is, the maximum of both values.} 

\begin{corollary}\label{cor}
For two communicating MDPs $\MDP,\MDP'$ that satisfy the assumptions of Lemma~\ref{lem:ProofPertBound} for 
time and state independent values $\Delta^r$, $\Delta^p$ (i.e., $\Delta^r_t(s)\leq \Delta^r$ and $\Delta^p_t(s)\leq \Delta^p$ for all $s$ and all $t$)
it holds that
\[
   \rho^*(M') - \rho^*(M)  \leq \Lambda' \Delta^p + \Delta^r. 
\]
\end{corollary}
\begin{proof}
%Assume without loss of generality that $\rho^*(M') \geq \rho^*(M)$.
 From Lemmata~\ref{lem:ProofPertBound} and \ref{lem:azuma-applied}
 we have that with probability $1-\delta$
\begin{align*}
&  \rho^*(\MDP') - \frac{1}{T} \sum_{t=1}^T \rup(s_t, \pi'^*(s_t))  \\
& \leq  \Lambda' \Delta^p +  \Delta^r + \frac{1}{T}\big(\Lambda' \sqrt{2\steps\log{(1/\delta)}} + \Lambda'\big) .  
\end{align*}
Choosing $\delta=1/T$ this yields for $T\to\infty$ and taking expectations that
\begin{align*}
    \rho^*(M') - \rho^*(M)  &\leq\, \rho^*(M') - \rho(M,\pi'^*)\\
    & \leq \Lambda' \Delta^p +  \Delta^r.\qedhere
\end{align*}  
\end{proof}

Corollary~\ref{cor} allows to give the following quick proof of Theorem~\ref{thm:PertBound}.
\begin{proof}[Proof of Theorem~\ref{thm:PertBound}]
Let $\Delta^r_t:=\max_{s,a} |\rup_{t+1}(s,a)-\rup_{t}(s,a)|$ and $\Delta^p_t:=\|p_{t+1}(\cdot|s,a)-p_{t}(\cdot|s,a)\|_1$.
Then by Corollary~\ref{cor} and the assumption that the diameters of all $M_t$ are bounded by $D$ we have
for $t=1,\ldots,T-1$
\[
   |\rho^*(M_{t+1})-\rho^*(M_t)| \leq D \Delta^p_t + \Delta^r_t
\]
and Theorem~\ref{thm:PertBound} follows by summing over all $t$.
\end{proof}

\subsection{OPTIMISM}
We show that the set of plausible MDPs with high probability contains each MDP $M_t$ the learner acts on in step~$t$. This is the theoretical justification for optimism, as it will allow us to show in the next section that the true reward can be upper bounded by the optimistic value $\rhot$.

\begin{lemma}\label{lem:optimism}
With probability $1-\frac{5\delta}{6}$, the set $\mathcal{M}(t)$ of plausible MDPs computed at any time step~$t$ contains all MDPs~$M_\tau$ for $\tau=1,\ldots,T$.
\end{lemma}
\begin{proof}
The proof is similar to the handling of failing confidence intervals for the colored UCRL algorithm given in Appendix A.2 of \cite{restless}. 

Fix a state-action pair $(s,a)$, and let $\tau_1, \tau_2,  \dots$ be the $N_{t}(s,a)$ time steps at which action~$a$ has been chosen in state~$s$, i.e., $(s_{\tau_i},a_{\tau_i})=(s,a)$ for all $i$. For the analysis of the transition probability estimates $\ph_t$ computed at step $t$ 
%let $$\sS_{s,a}:=\big\{s'\,|\,\exists \tau_i: p_{\tau_i}(s'|s,a)>0\big\}$$ 
%be the set of states $s'$ for which the probability of a transition at some step $\tau_i$ is $>0$.
we consider all vectors $\mathbf{x}$ indexed over the states with entries $\pm 1$. Then writing $x(s)$ for the entry in  $\mathbf{x}$ with index~$s$ we have
\begin{align*}
 & \big\| \ph_t(\cdot|s,a) - \mathbb{E}\big[\ph_t(\cdot|s,a)\big] \big\|_1 \\
    & = \sum_{s'}\Big| \ph_t(s'|s,a) - \mathbb{E}\big[\ph_t(s'|s,a)\big] \Big|  \\
    & \leq \!\! \max_{\mathbf{x}\in \{-1,1\}^{S}} \sum_{s'}\Big( \ph_t(s'|s,a) - \mathbb{E}\big[\ph_t(s'|s,a)\big] \Big) \, x(s') \\
    & \leq \max_{\mathbf{x}\in \{-1,1\}^{S}} \frac{1}{N_t(s,a)} \sum_{i=1}^{N_t(s,a)} X_i(\mathbf{x}),
\end{align*}
where we set 
\[
  X_i(\mathbf{x}) :=
    x(s_{\tau_i+1}) - \sum_{s'} p_{\tau_i}(s'|s_{\tau_i},a_{\tau_i})\, x(s').
\]
Now $\sum_i^{N_t(s,a)} X_i(\mathbf{x})$ is a martingale difference sequence with $|X_i(\mathbf{x})|\leq 2$
for any fixed $\mathbf{x}$ and fixed $N_t(s,a)=n$ so that by Azuma-Hoeffding (Lemma~\ref{azuma})
\begin{align*}
  \bP \Bigg\{\! \sum_{i=1}^{n} \! X_i(\mathbf{x}) \geq \sqrt{8 S n \log \big(\tfrac{8SAt^3}{\delta}\big) } \Bigg\} 
   %  \leq \Big(\frac{\delta}{8t^3SA}\Big)^{S}
     \leq  \frac{\delta}{8^S SA t^3}.
\end{align*}
A union bound over all $2^S$ vectors $\mathbf{x}$ and all possible values of $N_t(s,a)$
shows that with probability $1-\frac{\delta}{4SAt^2}$
\begin{equation}\label{eq:eqp1}
   \big\| \ph_t(\cdot|s,a) - \mathbb{E}\big[\ph_t(\cdot|s,a)\big] \big\|_1 
       \,\leq\,   \sqrt{\frac{8 S \log(8SAt^3/\delta)  }{\max(1,N_{t}(s,a))}}\,.
\end{equation}
Finally, note that for any fixed $N_t(s,a)$ we have 
\[
  \mathbb{E}\big[\ph_t(\cdot|s,a)\big]=\frac{1}{N_t(s,a)} \sum_{i=1}^{N_t(s,a)} p_{\tau_i}(\cdot|s,a),
\]
so that for all $\tau$
\[
    \big\| \mathbb{E}\big[\ph_t(\cdot|s,a)\big]  -  p_{\tau}(\cdot|s,a) \big\|_1  \leq V_T^p,
\]
which together with \eqref{eq:eqp1} shows that with probability $1-\frac{\delta}{4SAt^2}$ the transition probabilities $p_\tau(\cdot|s,a)$ at each step~$\tau$ are contained in the confidence intervals \eqref{eq:ConfProb} at step~$t$. 

For the rewards, $\rh_t(s,a) - \mathbb{E}\big[\rh_t(s,a)\big]$ as well as 
$\mathbb{E}\big[\rh_t(s,a)\big] - \rh_t(s,a)$ can be written as martingale difference sequences and Azuma-Hoeffding can be used to show that with probability $1-\frac{\delta}{4SAt^2}$,  the rewards $\rup_\tau(s,a)$ at each step $\tau$ are contained in the confidence intervals~\eqref{eq:ConfReward} for step~$t$. 

A union bound over all $t$ and all state-action pairs concludes the proof,
noting that $\sum_t \frac{\delta}{2t^2} \leq \tfrac{5\delta}{6}$.
 \end{proof}
 
In the following, we assume that the statement of Lemma~\ref{lem:optimism} holds, so that
we need to consider the respective error probability only once.

\subsection{$T$-STEP VS.\ AVERAGE REWARD}
In this section we consider the difference between the optimal $T$-step reward and the optimal average reward.
First we recall the well-known fact that the optimal $T$-step policy does not deviate by much from the optimal policy in respect to average reward,
see e.g.\ Exercise~38.17 of \cite{torcse}.

\begin{lemma}\label{lem:nonstationary}
 Let $M$ be a time-homogeneous and communicating MDP with diameter $D$ and rewards in $[0,1]$.
 Further let $v_T^*(M,s)$ be the optimal $T$-step reward when starting in state $s$. 
Then for any state~$s$,
\[ v_T^*(M,s) \,\leq\, T \rho^*(M) + D. \] 
\end{lemma}
Accordingly, the $T$-step reward in the changing MDP setting can also bounded by the optimistic average reward $\rhot$. 
\begin{lemma}\label{lem:optimism2}
Under the assumption of Lemma~\ref{lem:optimism}, for all~$k$ and all $s$,
$$v^*_T(s) \,\leq\, T \rhot_k + D.$$
\end{lemma}
\begin{proof}
Fix any $k$.
 As in Section 3.1 of \citep{jaksch} we consider the following extended MDP $\tilde{M}^+_k$ that corresponds to the set of plausible MDPs $\mathcal{M}_k$. That is, for any state~$s$ in $\sS$ the extended action set contains for each $a$ in the original action set~$\sA$, each value $\rt(s,a)$ in \eqref{eq:ConfReward}, and each transition probability distribution $\pt(\cdot|s,a)$ in \eqref{eq:ConfProb} a respective action with reward $\rt(s,a)$ and  transition probability distribution $\pt(\cdot|s,a)$. By the assumption that Lemma~\ref{lem:optimism} holds, the true values for rewards and transition probabilities at any step $\tau$ are contained in these confidence intervals, so that there is a nonstationary $T$-step policy on $\tilde{M}^+_k$ whose expected $T$-step reward is $v_T^*(s)$. Therefore, for the optimal nonstationary $T$-step reward on $\tilde{M}^+_k$, denoted by $\tilde{v}_T$, we have
\begin{equation}\label{eq:hit}
    \tilde{v}_T \,\geq\, v_T^*(s).
\end{equation}
Further, as $\tilde{M}^+_k$ contains the original transition probabilities of each $M_\tau$, its diameter is bounded by $D$. Hence, by Lemma~\ref{lem:nonstationary}, $\tilde{v}_T \leq  T \rho^*(\tilde{M}^+_k) + D.$ As $\rho^*(\tilde{M}^+_k)=\rhot_k$ (cf.\ Section 3.1 of \cite{jaksch}) this shows together with \eqref{eq:hit} the lemma.
\end{proof}

\subsection{REGRET WITHOUT RESTARTS}

As a next step, we derive the following regret bound for Variation-aware UCRL without restarts (i.e.\ Algorithm~\ref{alg}).

\begin{theorem}\label{thm:wr}
If the variation parameters are set to their true values, that is, $\Vt^r:=V_T^r$ and $\Vt^p:=V_T^p$, then after any $T$ steps the regret of Variation-aware UCRL without restarts (i.e., Algorithm~\ref{alg}) is upper bounded 
by
\[
   32 DS\sqrt{AT \log\big(\tfrac{8SAT^3}{\delta}\big)} + 2 T (V_T^r + DV_T^p)
\]
with probability $1-\delta$. This bound also holds when the algorithm starts in an initial state $s_1$ that is different from the initial state $s_1^*$ of the optimal $T$-step policy we compare to. 
\end{theorem}

For the proof we will use the following two basic facts about UCRL (Proposition~18 and Lemma~19 of \cite{jaksch}), which can be derived from its episode termination criterion. As we use the same criterion these results also hold for our variation-aware adaptation.

\begin{lemma}\label{lem:no-episodes}
The number of episodes $K$ of Variation-aware UCRL after any $T>SA$ steps is bounded by $SA\log_2\big(\tfrac{T}{SA}
\big)$.
\end{lemma}

\begin{lemma}\label{lem:vk/Nk}
\[
   \sum_{s,a} \sum_{k=1}^K  \frac{v_k(s,a)}{\sqrt{\max(1,N_k(s,a))}} \,\leq\, \big(\sqrt{2}+1\big) \sqrt{SAT}.
\]
\end{lemma}

\begin{proof}[Proof of Theorem~\ref{thm:wr}]
First, let $\rhot_{\min}:=\min_k \rhot_k$. Then denoting the length of episode~$k$ by $T_k:=t_{k+1}-t_k$ (setting $t_{K+1}:=T$), by Lemma~\ref{lem:optimism2} we have 
\begin{equation}\label{eq:p1}
  v^*_T(s_1^*) \,\leq\, T \rhot_{\min} + D \,\leq\, \sum_{k=1}^K T_k \rhot_k + D.  
\end{equation}
Further, another application of Azuma-Hoeffding (Lemma~\ref{azuma}) shows that for the rewards $r_t$ collected by the algorithm with probability $1-\frac{\delta}{12}$
\begin{equation}\label{eq:p2}
   \sum_{t=1}^T \big( \rup(s_t,a_t) - r_t \big) 
     \, \leq \, \sqrt{2 T \log \big(\tfrac{12}{\delta}\big)}.
\end{equation}
Combining \eqref{eq:p1} and \eqref{eq:p2} the regret is bounded as
\begin{align}
    & R_T \,=\, v^*_T(s_1^*) - \sum_{t=1}^T r_t   \nonumber \\
   \quad &  =  v^*_T(s_1^*)  -  \sum_{t=1}^T \rup(s_t,a_t)  
          +  \sum_{t=1}^T \big( \rup(s_t,a_t) - r_t \big)  \nonumber \\
   \quad  & \,\leq\, \sum_{k=1}^K  \Big( T_k \rhot_k - \sum_{t=t_k}^{t_{k+1}-1} \rup(s_t,\pit_k(s_t)) \Big)
            \nonumber \\
          & \qquad + \sqrt{2 T \log \big(\tfrac{12}{\delta}\big)} + D .
            \label{eq:r1}
\end{align}
Now we are going to bound the term in the sum for each episode~$k$ using Lemma~\ref{lem:ProofPertBound}. Indeed, we perform the optimal policy $\pit_k$ of the optimistic MDP $\tilde{M}_k$ with average reward $\rhot_k$ on the underlying (non-time-homogeneous) true MDP, and the rewards $\rt_k$ and the transition probabilities $\pt_k$ of $\tilde{M}_k$ satisfy the confidence intervals \eqref{eq:ConfReward} and \eqref{eq:ConfProb} according to Lemma~\ref{lem:optimism} so that
\begin{align*}
  & \big|\rup_t(s,a) - \rt_k(s,a)  \big| \\
  &\quad \leq  \big|\rup_t(s,a) - \rh_k(s,a)  \big| + \big|\rh_k(s,a) - \rt_k(s,a)  \big|  \\
  &\quad \leq  V^r_T + \Vt^r + 2 \sqrt{\tfrac{8\log{(8SAt_k^3/\delta)}}{\max{(1, N_k(s,a))}}}
\end{align*}
as well as
\begin{align*}
  &   \big\|p_t(\cdot|s,a) - \pt_k(\cdot|s,a) \big\|_1 \\
   &  \!\!\leq  \big\|p_t(\cdot|s,a) - \ph_k(\cdot|s,a)  \big\|_1 \!+ \big\|\ph_k(\cdot|s,a) - \pt_k(\cdot|s,a)  \big\|_1  \\
  & \!\!\leq  V^p_T + \Vt^p + 2 \sqrt{\tfrac{8S\log{(8SAt_k^3/\delta)}}{\max{(1, N_k(s,a))}}} .
\end{align*}

By assumption $\Vt^r= V_T^r$ and $\Vt^p= V_T^p$, and 
 Lemma~\ref{lem:ProofPertBound} gives
\begin{align}
 &  T_k \rhot_k\, - \sum_{t=t_k}^{t_{k+1}-1}\! \rup(s_t,\pit_k(s_t)) \nonumber \\
 & \leq \, 2\tilde{\Lambda}_k\! \sum_{t=t_k}^{t_{k+1}-1} \Big(V^p_T + \sqrt{\tfrac{8S\log{(8SAt_k^3/\delta)}}{\max{(1, N_k(s_t,\pit_k(s_t)))}}} \Big) \label{eqx1}\\
   & \quad + 2\sum_{t=t_k}^{t_{k+1}-1} \Big( V^r_T + \sqrt{\tfrac{8\log{(8SAt_k^3/\delta)}}{\max{(1, N_k(s_t,\pit_k(s_t)))}}} \Big) \label{eqx2} \\
& \quad + \!\! \sum_{t=t_k}^{t_{k+1}-1} \!\! \Big(\sum_{s'} p_t(s'|s_t,\pit_k(s_t))\cdot \tilde{\lambda}_k(s') - \tilde{\lambda}_k(s_t) \Big),  \label{eqx3}
\end{align}
where $\tilde{\lambda}_k$ is the bias function of $\pit_k$ on $\tilde{M}_k$ and $\tilde{\Lambda}_k$ is the respective bias span. Since by Lemma~\ref{lem:optimism} the set of plausible MDPs contains each MDP $M_t$
and the diameter of each $M_t$ is bounded by $D$, the bias span $\tilde{\Lambda}_k \leq D$, cf.\ \citep{jaksch,regal}.\footnote{Note that for the argument it would be sufficient if just one of the MDPs $M_t$ is plausible. However, for the restart scheme of Algorithm~\ref{alg2} we would need a plausible $M_t$ in each phase.}

Accordingly, the sum of \eqref{eqx1} and \eqref{eqx2} over all episodes is bounded by Lemma~\ref{lem:vk/Nk}  as
\begin{align}
% & \sum_{k=1}^K \Big( T_k \rhot_k - \sum_{t=t_k}^{t_{k+1}-1} \rup(s_t,\pit_k(s_t)) \Big) \nonumber \\
 &  \!\!\!\!\!\! \leq \, 2T (D V^p_T + V^r_T ) \nonumber\\
  & \!\!\!\!  + 2(D+1) \sqrt{8S\log\!\big(\tfrac{8SAT^3}{\delta}\big)} \sum_{s,a} \sum_{k=1}^K \tfrac{v_k(s,a)}{\sqrt{\max{(1, N_k(s,a))}}} \nonumber \\
   & \!\!\!\!\!\! \leq \, 2T (D V^p_T + V^r_T ) \nonumber \\
   & \quad + 2 \big(\sqrt{2}+1\big) (D+1)S \sqrt{8AT\log\big(\tfrac{8SAT^3}{\delta}\big)}.
% & \quad  + \sum_{t=1}^{T}  \Big(\sum_{s'} p_t(s'|s_t)\cdot \tilde{\lambda}(s') - \tilde{\lambda}(s_t) \Big). \label{eq:r2}
\end{align}
The sum in \eqref{eqx3} can be written as in the proof of Lemma~\ref{lem:azuma-applied} as
\begin{align*}
  & \sum_{t=t_k}^{t_{k+1}-1}  \Big(\sum_{s'} p_t(s'|s_t,\pit_k(s_t))\cdot \tilde{\lambda}_k(s') - \tilde{\lambda}_k(s_t) \Big)  \\
  & =  \sum_{t=t_k}^{t_{k+1}-1}  \Big(\sum_{s'} p_t(s'|s_t,\pit_k(s_t))\cdot \tilde{\lambda}_k(s') - \tilde{\lambda}_k(s_{t+1}) \Big) \\
  & \qquad  + \tilde{\lambda}_k(s_{t_{k+1}}) - \tilde{\lambda}_k(s_{t_{k}}) .
\end{align*}
Now $\tilde{\lambda}_k(s_{t_{k+1}}) - \tilde{\lambda}_k(s_{t_{k}})\leq D$, so that summing over all episodes gives by Azuma Hoeffding (Lemma~\ref{azuma}) and Lemma~\ref{lem:no-episodes} that with probability $1-\frac{\delta}{12}$
\begin{align}
&\sum_{k=1}^K  \sum_{t=t_k}^{t_{k+1}-1}  \Big(\sum_{s'} p_t(s'|s_t,\pit_k(s_t))\cdot \tilde{\lambda}_k(s') - \tilde{\lambda}_k(s_t) \Big) \nonumber \\
  & \leq  \sum_{t=1}^{T} \Big(\sum_{s'} p_t(s'|s_t,\pit_k(s_t))\cdot \tilde{\lambda}_{k(t)}(s') - \tilde{\lambda}_{k(t)}(s_{t+1}) \Big) \nonumber \\ & \qquad + KD \nonumber \\ 
  &  \leq D \sqrt{2T\log{\big(\tfrac{\delta}{12}\big)}} + D SA \log_2 \big(\tfrac{T}{SA}\big),
  \label{eq:r3}
\end{align}
where $k(t)$ denotes the episode in which time step $t$ occurs.

Thus, combining \eqref{eq:r1}--\eqref{eq:r3} taking into account the error probabilities for
\eqref{eq:p2},  \eqref{eq:r3}, and Lemma~\ref{lem:optimism},
 we yield that with probability $1-\delta$ the regret is bounded by
\begin{align*}
  & R_T \, \leq \, \sqrt{2 T \log \big(\tfrac{12}{\delta}\big)} + D + 2T (D V^p_T + V^r_T )  \\
  & \quad  +  2 \big(\sqrt{2}+1\big) (D+1)S \sqrt{8AT\log\big(\tfrac{8SAT^3}{\delta}\big)}. \\ 
  & \quad +  D \sqrt{2T\log{\big(\tfrac{12}{\delta}\big)}} + D SA \log_2 \big(\tfrac{T}{SA}\big) 
\end{align*}
and some simplifications analogous to Appendix C.4 of \cite{jaksch} give the claimed regret bound.
\end{proof}

\subsection{PROOF OF THEOREM \ref{thm:vregret-ucrl-r}}
\label{sec:ProofMainThm}
Finally, we are ready to give the proof of the regret bound for the restart scheme of Algorithm~\ref{alg2}.
Abusing notation we write $V_i^r$ and $V_i^p$ for the variation of rewards and transition probabilities in phase $i$ and abbreviate $V_i:=V_i^r+V_i^p$, $V:= V_T^r+V_T^p$ and $\theta_i:= \big\lceil\frac{i^2}{V^2}\big\rceil$.

First, let us bound the number of phases $N$.
 Obviously, step $T$ is reached in phase $N$ when 
\begin{align*}
  &   \sum_{i=1}^{N-1} \Big\lceil\frac{i^2}{V^2}\Big\rceil < T   \leq
      \sum_{i=1}^{N} \Big\lceil\frac{i^2}{V^2}\Big\rceil .
\end{align*}
Recalling that $\sum_{i=1}^N i^2 = \frac{1}{6}N(N+1)(2N+1) > \frac{1}{3}N^3$
we obtain 
\begin{align*}
  &  T > \sum_{i=1}^{N-1} \Big\lceil\frac{i^2}{V^2}\Big\rceil 
        >  \sum_{i=1}^{N-1}  \frac{i^2}{V^2}
         >  \frac{(N-1)^3}{3V^2},
\end{align*}
so that the number of phases is bounded as
\begin{align}\label{no-phases}
 & N < 1+ \sqrt[3]{3V^2T}.
\end{align}
Writing $\tau_i$ for the initial step of phase~$i$ and $s^*_{\tau_i}$ for the (random) state 
visited by the optimal $T$-step policy at step $\tau_i$, we can decompose the regret as
\begin{equation}\label{eq:sparks}
   v^*_T(s_1)- \sum_{t=1}^T r_t  
    = \sum_{i=1}^N \Big( \mathbb{E}\big[ v^*_{\theta_i}(s^*_{\tau_i})\big] 
         - \sum_{t=\tau_i}^{\tau_i-1} r_t \Big).
\end{equation}
By Theorem~\ref{thm:wr} and a union bound over all possible values for state~$s^*_{\tau_i}$, the $i$-th summand ($i=1,\ldots,N$) in \eqref{eq:sparks} with probability $1-\frac{\delta}{2\tau_i^2}$ is bounded by
\begin{align*}
  & 32 DS\sqrt{A\log\big(\tfrac{16S^2AT^5}{\delta}\big)} \cdot \sqrt{\theta_i}
        \,+\, 2 D V_i \cdot \theta_i .   
\end{align*}
If $\sqrt[3]{3V^2T} < 1$, then we also have $3V^2T < 1$ and hence $3V^2T^2 < T$, so that 
\[
   VT < \sqrt{3} \cdot VT < \sqrt{T} .
\]
Further, in this case by \eqref{no-phases} we have $N=1$ with $\theta_1=T$
and $V_1=V$,
so that the regret is bounded by 
\begin{align*}
  &  32 DS\sqrt{A\log\big(\tfrac{16S^2AT^5}{\delta}\big)} \cdot \sqrt{T} + 2 D V T   \\
   &  < \, \Big(32 DS\sqrt{A\log\big(\tfrac{16S^2AT^5}{\delta}\big)} + 2D \Big) \cdot \sqrt{T} ,
\end{align*}
which is upper bounded by the claimed regret bound.

On the other hand, if $\sqrt[3]{3V^2T} \geq 1$, then $N<2\sqrt[3]{3V^2T}$ from 
\eqref{no-phases} and summing over all $N$ phases yields from \eqref{eq:sparks} that with error probability 
$\sum_i \frac{\delta}{2\tau_i^2} < \sum_t \frac{\delta}{2t^2} < \delta$
the regret is bounded by
\[
   32 DS\sqrt{\!A\log\!\big(\!\tfrac{16S^2AT^5}{\delta}\!\big)} \cdot \sum_{i=1}^N \!\sqrt{\theta_i} 
     \, +\, 2 D \sum_{i=1}^N \! V_i \Big(\frac{i^2}{V^2} + 1\Big).
\]
\iffalse
\begin{align*}
  & 32 DS\sqrt{A\log\big(\tfrac{16SAT^5}{\delta}\big)} \cdot \sum_{i=1}^N \sqrt{\theta_i} 
      + \sum_{i=1}^N V_i \Big(\frac{i^2}{V^2} + 1\Big) \nonumber  \\
 %  & \quad \leq 32  DS\sqrt{A\log\big(\tfrac{16SAT^5}{\delta}\big)} \sqrt{NT}  
 %     + \sum_{i=1}^N V_i \Big(\frac{N^2}{V^2} + 1\Big)  \nonumber  \\
 %  & \quad \leq 32  DS\sqrt{A\log\log\big(\tfrac{16SAT^5}{\delta}\big)} \sqrt{NT}  
 %     +  \frac{N^2}{V} + V  \nonumber \\
 %  & \quad \leq 32 DS\sqrt{A\log\log\big(\tfrac{16SAT^5}{\delta}\big)} V^{1/3} T^{2/3}  + c_3 V^{1/3} T^{2/3}  +  V,
\end{align*}
\fi
%
Noting that using Jensen's inequality
\[
    \sum_{i=1}^N \sqrt{\theta_i} \,\leq\,  \sqrt{NT}  \,\leq\,  1.7 \cdot V^{1/3} T^{2/3}
\]
and that also
\begin{align*}
   &  \sum_{i=1}^N V_i \Big(\frac{i^2}{V^2} + 1\Big)  
        \,\leq\, \sum_{i=1}^N V_i \Big(\frac{N^2}{V^2} + 1\Big)  
         \,\leq\,  \frac{N^2}{V} + V    \\
       & < \, 8.33 \cdot V^{1/3} T^{2/3}  +  V,
\end{align*}
concludes the proof, noting that the claimed bound holds trivially if $V\geq T$, so that we may assume that $V<T$ and hence $V<V^{1/3} T^{2/3}$.
\qed

\section{DISCUSSION AND EXTENSIONS}
\label{sec:Discussion}
%%%%

The regret bound of Theorem~\ref{thm:wr} relies on the assumption that the variation for rewards and transition probabilities are known in advance. Accordingly it is necessary for the restart scheme 
to know the respective variation terms for each single phase.
It is easy to check that if upper bounds on these values are used to set $\Vt^r$ and $\Vt^p$ instead, the regret bounds of Theorems~\ref{thm:regret-ucrl-r} and \ref{thm:wr} simply depend on these upper bounds instead of the true values. 

In principle, it is also possible to set the variation parameters $\Vt^r$ and $\Vt^p$ in Algorithm~\ref{alg} to 0.
Then Lemma~\ref{lem:optimism} need not hold anymore, that is, it is not guaranteed that the set of plausible MDPs contains any of the MDPs $M_t$. Accordingly, we cannot rely on Lemma~\ref{lem:optimism2} anymore, which is based on Lemma~\ref{lem:optimism} and guarantees that the optimistic average reward is an upper bound on the true reward. However, taking into account the true variation one can still establish an upper bound on the true reward.

\begin{lemma}\label{lem:alt-optimism}
Let $\rhot$ be the optimistic average reward computed when using the true variations $V^r_T$ and $V^p_T$ and $\rhot^0$~be the optimistic average reward computed with variation parameters $\Vt^r=0$ and $\Vt^p=0$. Then 
\[
   \rhot \,\leq \,   \rhot^0 + V^r_T + D V^p_T.
\]
\end{lemma}

\begin{proof}
The rewards and transition probabilities of the respective optimistic MDPs $\tilde{M}$ and $\tilde{M}^0$ 
with $\rhot=\rho^*(\tilde{M})$ and $\rhot^0=\rho^*(\tilde{M}^0)$ differ by at most $V^r_T$ and $V^p_T$, respectively. Hence the claim of the lemma follows by  Corollary~\ref{cor}, recalling that the bias span of the optimal policy in $\tilde{M}$ is bounded by the diameter $D$, cf.\ the proof of Theorem~\ref{thm:wr}.
\end{proof}

Thus, in principle one could use Lemma~\ref{lem:alt-optimism} to replace Lemma~\ref{lem:optimism2} in the proof of Theorem~\ref{thm:wr}. 
It is easy to check that this works fine except for the second application of Lemma~\ref{lem:optimism} used to show that the bias span $\tilde{\Lambda}$ is bounded by $D$. Indeed, the set of plausible MDPs $\mathcal{M}^0$ computed with $\Vt^r=0$ and $\Vt^p=0$ need not contain any of the MDPs $M_t$ and hence there is no guarantee that it contains an MDP with diameter bounded by~$D$.

Still it is possible to obtain an alternative bound on the bias span by observing that $\mathcal{M}^0$ contains with high probability an MDP where for each state-action pair $(s,a)$
there is a subset $\mathcal T$ of $\{1,2,\ldots,T\}$ such that the transition probabilities under $(s,a)$ are of the form
\begin{equation}\label{eq:pro}
     p(\cdot|s,a) =  \frac{1}{|\mathcal T|} \sum_{t\in \mathcal T} p_t(\cdot|s,a).
\end{equation}
The intution here is that the set $\mathcal T$ corresponds to the time steps when a sample for the state-action pair $(s,a)$ has been taken.
Let $\hat{\mathcal{M}}$ be the set of MDPs with transition probabilities of the form specified in \eqref{eq:pro}. Then defining 
\[
    \hat{D} := \max_{M\in \hat{\mathcal{M}}} D(M)
\]
to be the maximal diameter over all MDPs in $\hat{\mathcal{M}}$, one has $\tilde{\Lambda} \leq \hat{D}$
and the following regret bound holds when setting the variation parameters to 0.

\begin{theorem}\label{thm:app1}
With probability $1-\delta$, the regret of Variation-aware UCRL with variation parameters set to~0 
is upper bounded by
\[
    R_T \leq 32 \hat{D}S\sqrt{AT \log\big(\tfrac{8SAT^3}{\delta})} +  2 T (V_T^r + DV_T^p).
\]

\end{theorem}

Using Theorem~\ref{thm:app1}, one can derive the following regret bound for the restart scheme of Algorithm~\ref{alg2}. Note that the restart scheme still needs $V^r_T + V^p_T$ as input.

\begin{theorem}\label{thm:app2}
With probability $1-\delta$, the regret of the restart scheme of Algorithm \ref{alg2} with variation parameters set to~0 in each phase is bounded by
\begin{align*}
   R_T \leq 74 \cdot (V_T^r + V_T^p)^{1/3}\, T^{2/3} \hat{D} S \sqrt{A\log{\big(\tfrac{16S^2AT^5}{\delta})}} .
\end{align*}
\end{theorem}

Similar bounds obviously also hold when the variation parameters are set to any value smaller than the true variation. Note that $D \leq \hat{D}$, as the set $\hat{\mathcal{M}}$ also contains the MDPs $M_t$ for each $t$.
However, the following example shows that $\hat{D}$ in general cannot be bounded by $D$ and in some even simple cases can be unbounded.

\textbf{Example.}
Consider two MDPs $M_1,M_2$ over the same state space $\{s,s'\}$ and the same action space $\{a,a'\}$.
In $M_1$ the nonzero transition probabilities are given by $p_1(s|s',a)=1$, $p_1(s'|s,a)=1/D$, $p_1(s|s,a)=1-1/D$, and $p_1(s|s,a')=p_1(s'|s',a')=1$. In $M_2$ the roles of $a$ and $a'$ are swapped, that is, we have $p_2(s|s,a)=p_2(s'|s',a)=1$, $p_2(s|s',a')=1$, $p_2(s'|s,a')=1/D$, and $p_2(s|s,a')=1-1/D$.
Obviously, both MDPs have diameter $D$. However, the MDP~$M$ with nonzero transition probabilities 
$p(s|s,a):=p_2(s|s,a)=1$, $p(s'|s',a):=p_2(s'|s',a)=1$, 
$p(s|s,a'):=p_1(s|s,a')=1$, and $p(s'|s',a'):=p_1(s'|s',a')=1$
is contained in $\hat{\mathcal{M}}$, but does not have finite diameter, as the states $s$, $s'$ are not connected.

%\section{CONCLUSION}\label{sec:conc}
To conclude, we note that recently variational bounds for the (contextual) bandit setting have been derived also for the case when the variation is unknown \citep{luo19}. Achieving such bounds in our setting seems not easy, as sampling a particular state-action pair usually causes some transition costs.
\medskip

\textbf{Acknowledgements.}
This work has been supported by the Austrian Science Fund (FWF): I 3437-N33
in the framework of the CHIST-ERA ERA-NET (DELTA project).
%Ronald Ortner would like to thank his family for their patience when preparing the
%final version of this paper as well as the Sparks and Giorgio Moroder for making
%``N$^{\underline{\mbox{o}}}$1 in Heaven''.

%\newpage
\bibliographystyle{plainnat}

\end{document}